\documentclass{ecai} 

\usepackage{latexsym}
\usepackage{amssymb}
\usepackage{amsmath}
\usepackage{amsthm}
\usepackage{booktabs}
\usepackage{enumitem}
\usepackage{graphicx}
\usepackage{color}

\usepackage{mathtools}
\usepackage{algorithm}
\usepackage{algorithmic}

\usepackage{amsthm}

\usepackage{color}
\usepackage{amsfonts}
\usepackage{dsfont}
\usepackage{multirow}
\usepackage[table,xcdraw]{xcolor}
\usepackage{tabularx}
\usepackage{subfigure}
\usepackage{booktabs}
\usepackage{makecell}

\newcommand{\smalltitle}[1]{ \vspace{1mm}{\noindent\textbf{#1.}\hspace{1mm}}}

\newtheorem{theorem}{Theorem}

\newtheorem{definition}{Definition}
\newtheorem{example}{Example}

\newcommand{\BibTeX}{B\kern-.05em{\sc i\kern-.025em b}\kern-.08em\TeX}

\begin{document}


\begin{frontmatter}


\paperid{123} 


\title{Lossless Token Merging Even Without Fine-Tuning in Vision Transformers}




\author[A]{\fnms{Jaeyeon}~\snm{Lee}}
\author[A]{\fnms{Dong-Wan}~\snm{Choi}\thanks{Corresponding Author. Email: dchoi@inha.ac.kr}}
\address[A]{Department of Computer Science and Engineering, Inha University, South Korea}


\begin{abstract}
Although Vision Transformers (ViTs) have become the standard architecture in computer vision, their massive sizes lead to significant computational overhead. Token compression techniques have attracted considerable attention to address this issue, but they often suffer from severe information loss, requiring extensive additional training to achieve practical performance. In this paper, we propose \textbf{A}daptive \textbf{T}oken \textbf{M}erging (ATM), a novel method that ensures lossless token merging, eliminating the need for fine-tuning while maintaining competitive performance. ATM adaptively reduces tokens across layers and batches by carefully adjusting layer-specific similarity thresholds, thereby preventing the undesirable merging of less similar tokens with respect to each layer. Furthermore, ATM introduces a novel token matching technique that considers not only similarity but also merging sizes, particularly for the final layers, to minimize the information loss incurred from each merging operation. We empirically validate our method across a wide range of pretrained models, demonstrating that ATM not only outperforms all existing training-free methods but also surpasses most training-intensive approaches, even without additional training. Remarkably, training-free ATM achieves over a 30\% reduction in FLOPs for the DeiT-T and DeiT-S models without any drop in their original accuracy.
\end{abstract}





\end{frontmatter}


\section{Introduction}
Vision Transformers (ViTs) have reported state-of-the-art performance across a wide range of visual recognition tasks \cite{LiuL00W0LG21,DosovitskiyB0WZ21}. This superior performance, however, essentially comes with their massive sizes, which suffer from not only prolonged training time but also high computational costs. Although various model compression techniques \cite{0004HWCCC22,ZhenglZYTXRP22} have attempted to reduce the size of ViT itself, they commonly make permanent modifications to the model architecture, potentially losing well-pretrained prior knowledge as well as often requiring additional training steps for the performance recovery \cite{abs-2110-04869}. 

\begin{figure}[t]
    \centering
    \includegraphics[width=0.9\columnwidth]{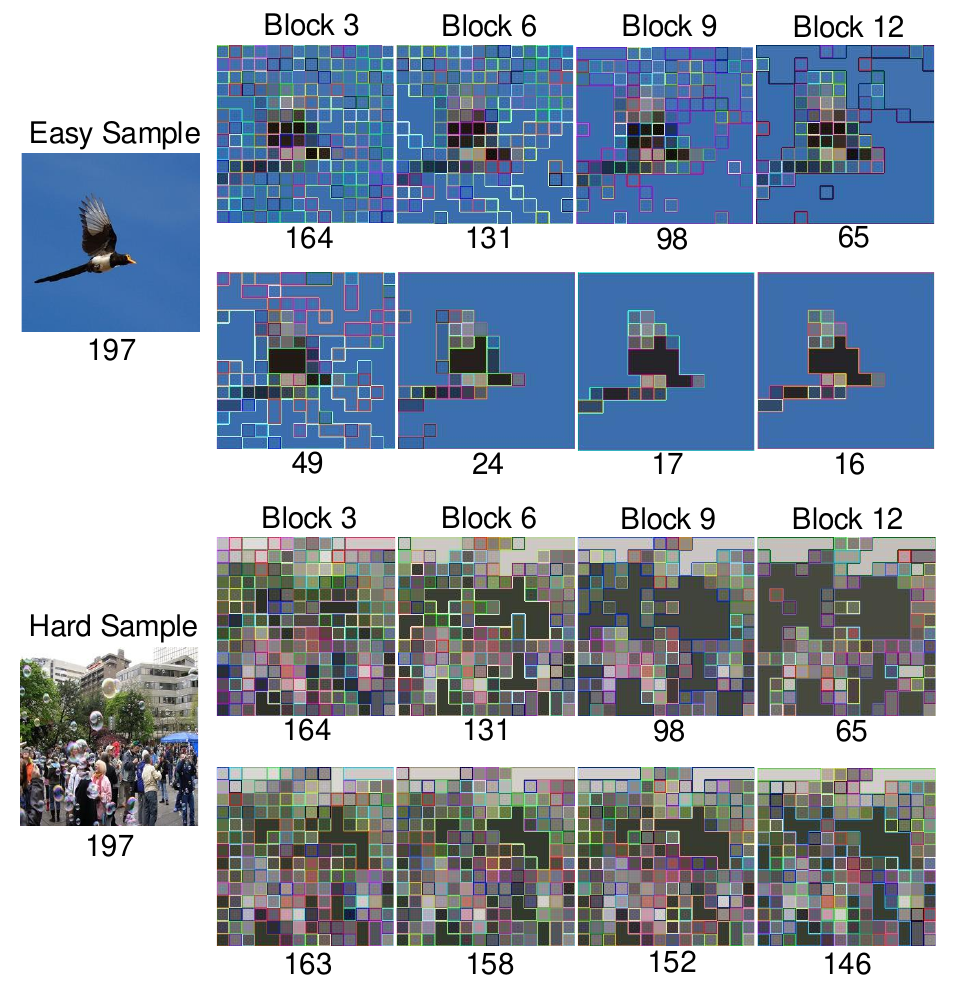}
    \caption{\textbf{Comparison of static and adaptive token merging methods.} For each sample image, the upper row presents the visualization of static top-$r$ merging \protect\cite{BolyaFDZFH23}, while the lower row shows the results of threshold-based adaptive merging. The number below each image indicates the count of remaining tokens after the corresponding merging process.}
    \label{fig1}
\end{figure}

This has brought a recent alternative, namely \textit{token compression}, which aims to dynamically reduce the number of input tokens along the inference process. Without requiring to change the model architecture, token compression has been reported to effectively improve the computational efficiency of ViTs. The initial approach is \textit{token pruning} \cite{RaoZLLZH21,PanPJWFO21}, aiming to identify and prune unimportant tokens with respect to some heuristic criterion, such as attention scores of the classification (CLS) tokens \cite{LiangGTS0X22,XuZZSLDZXS22}. To further mitigate information loss of pruning, \textit{token merging} \cite{BolyaFDZFH23,Leeck24} is being studied as a promising solution, which focuses on merging similar token pairs, rather than just completely discarding less important tokens.

Unfortunately, most existing token merging methods do not fully enjoy the dynamic nature of token compression, as they define a fixed number of pairs to be merged for each layer and for each instance, such as $r$ most similar pairs, a.k.a. \textit{top-$r$ token merging} \cite{BolyaFDZFH23,KimGHSJ24}. However, this static approach is far from optimal, as the degree of token similarity can vary across layers as well as input samples. As visualized in Figure \ref{fig1}, top-$r$ token merging reduces the same number of tokens per layer per instance, which results in under-merging for easy samples yet over-merging for hard samples. This leads to either inefficient or excessively aggressive token reduction, often requiring expensive fine-tuning or even training from scratch (e.g., 300 epochs in \cite{LongZP0023}) to recover the original performance. This expensive post-training negates the practical motivation for using token compression in the sense that model compression methods also show impressive performance, given the support of extensive fine-tuning.


In this paper, we address these limitations by proposing \textit{training-free} and \textit{adaptive} token compression, called \textbf{A}daptive \textbf{T}oken \textbf{M}erging (\textbf{ATM}), which completely eliminates the need for fine-tuning in token merging while ensuring high performance. To this end, ATM proposes \textit{threshold-based merging} that can adaptively vary the number of token pairs to merge across instances and layers, as shown in Figure \ref{fig1}, rather than merging a fixed number of pairs with the highest similarity. To properly set thresholds that minimize information loss, we introduce \textit{layer-dependent thresholding}, based on our analysis of the impact of token merging across different layers. Furthermore, aiming to reduce the information loss incurred from each merging operation, ATM introduces \textit{size-distinctive token matching} particularly for the final layers, which takes into account the accumulated number of previously merged tokens (i.e., \textit{merging size}) when grouping and matching tokens to merge. Lastly, to support practical batched inference, ATM incorporates \textit{batch-adaptive merging}, which enables adaptive token compression in batch-processing scenarios.

Through extensive experiments, our ATM method demonstrates superior performance across various pretrained models without the help of any extra training.
Notably, ATM manages to reduce FLOPs by over 30\% in both DeiT-T and DeiT-S without any accuracy drop from their original performance.
Somewhat surprisingly, this training-free performance already surpasses those of existing methods that take advantage of additional training.
Even in the performance comparison with fine-tuning, our method is not only able to further improve the performance but also it reveals high training efficiency, thanks to the reduced information loss in our merging scheme.

\section{Related Works}

\smalltitle{Efficient Vision Transformers}
A wide range of research has focused on enhancing the efficiency of ViTs while maintaining their performance, either by designing efficient ViT architectures \cite{WangX0FSLL0021,HeoYHCCO21} or by applying compression techniques, such as weight pruning \cite{ZhenglZYTXRP22} and knowledge distillation \cite{WuZPLXFY22}. Without altering model architectures, token compression has emerged as another effective approach that accelerates inference speed by reducing the number of input image tokens. 
Unlike the former approaches, token compression can often be applied without additional training, practically more appealing to various scenarios of improving ViT efficiency.


\smalltitle{Token Compression}
Given that the computational complexity of ViTs scales quadratically with the number of input tokens, numerous studies have explored token compression strategies. 
The vast majority of these have focused on two approaches, namely \textit{token pruning} and \textit{token merging}.
The former primarily utilizes CLS token's attention scores \cite{XuZZSLDZXS22,LiangGTS0X22} or additional learnable parameters \cite{RaoZLLZH21,MengLCLWJL22} to distinguish unimportant tokens, while the latter considers diversity \cite{MarinCRPRT23,LongZP0023} or similarity between token pairs \cite{BolyaFDZFH23} to identify the tokens that could be summarized into one representative. Furthermore, recent works combine pruning and merging \cite{KimGHSJ24}, or select tokens based on multiple criteria such as similarity, informativeness, and merging size \cite{Leeck24}.
However, most of these methods suffer from significant performance degradation without the help of additional post-training. One underlying reason lies in the lack of adaptive efficiency in their reduction scheme that consistently removes the same number of tokens for each layer.



\smalltitle{Training-Free Token Compression} 
A few approaches in token compression have attempted to eliminate the need of additional fine-tuning, by solely utilizing the model architecture and tokens themselves. ATS \cite{FayyazKJSJSPG22} prunes tokens using inverse transform sampling after obtaining importance scores based on the CLS token's attention scores and the value vector norms.
ToMe \cite{BolyaFDZFH23} simply merges the most similar top-\textit{r} token pairs for each layer, while MCTF \cite{Leeck24} takes into account multiple criteria to select token pairs to be merged.
Zero-TP \cite{WangDj24} uses the attention matrix to prune the tokens using importance-based and similarity-based metrics.
However, none of the existing training-free approaches get to achieve the same level of performance as training-required counterparts, due to severe information loss incurred from their token reduction strategies.




\section{Adaptive Token Merging} 
In this section, we first provide an overview of token compression in the ViT structure, and then introduce our proposed ATM method in detail.

\subsection{Preliminaries}
\smalltitle{Token Reduction in ViT}
In the ViT architecture \cite{0007CWYSJTFY21,JiangHYZSJWF21}, multiple blocks are stacked sequentially and the input tokens $\mathbf{X}^l \in \mathbb{R}^{N \times D}$ are processed through these blocks, where $l \in [1, L]$ is the block index, $L$ is the total number of blocks, $N$ is the number of tokens, and $D$ is the dimension of each token.
Within each block, the input $\mathbf{X}^l$ is processed as follows:
\begin{align}
    \mathbf{X}_{\text{Attn}}^l &= \mathbf{X}^l + \text{Attn}(\text{LN}(\mathbf{X}^l)), \nonumber \\
    \mathbf{X}^{l+1} &= \mathbf{X}_{\text{Attn}}^l + \text{MLP}(\text{LN}(\mathbf{X}_{\text{Attn}}^l)), \nonumber
\end{align}
where $\text{Attn}$, $\text{MLP}$, and $\text{LN}$ represent multi-head self-attention, multi-layer perceptron, and layer normalization, respectively.
Since the computational cost increases quadratically with respect to $N$ in $\text{Attn}$ and linearly in $\text{MLP}$, the goal of token compression is minimizing $N$ as much as possible to mitigate this $N$-dependent cost growth.


\smalltitle{Baseline Token Merging}
In the existing token merging methods \cite{BolyaFDZFH23,KimGHSJ24,Leeck24}, a common approach is to (i) divide the tokens for each layer into two equal-sized groups by alternating token assignment and (ii) perform a bipartite matching procedure for merging similar token pairs within the layer. Although this bipartite matching scheme has reported to incur relatively low information loss, compared to token pruning, all existing methods take a \textit{static approach} that merges a predefined number of token pairs across layers and images, leading to either over-merging or under-merging as observed in Figure \ref{fig1}.





\subsection{Threshold-Based Adaptive Token Merging}
Aiming to address limitations of baseline token merging, we revisit existing static merging scheme, thereby proposing an adaptive token merging method, called ATM, which takes into account both image-wise and layer-wise adaptation. Our basic strategy is \textit{threshold-based merging}, in which we merge token pairs that are truly similar enough based on the threshold for each layer, instead of a fixed number of pairs with the highest similarity. As presented in Figure \ref{fig1}, this threshold-based merging effectively controls the degree of token reduction depending on the complexity of the image features. An important issue here is how to properly set the threshold to minimize the loss of information. An easy yet extreme solution could be globally applying a single threshold to all layers, but this is not likely to achieve satisfactory performance, as analyzed in our empirical study as follows.

\begin{figure}[t]
    \centering
    \includegraphics[width=0.9\columnwidth]{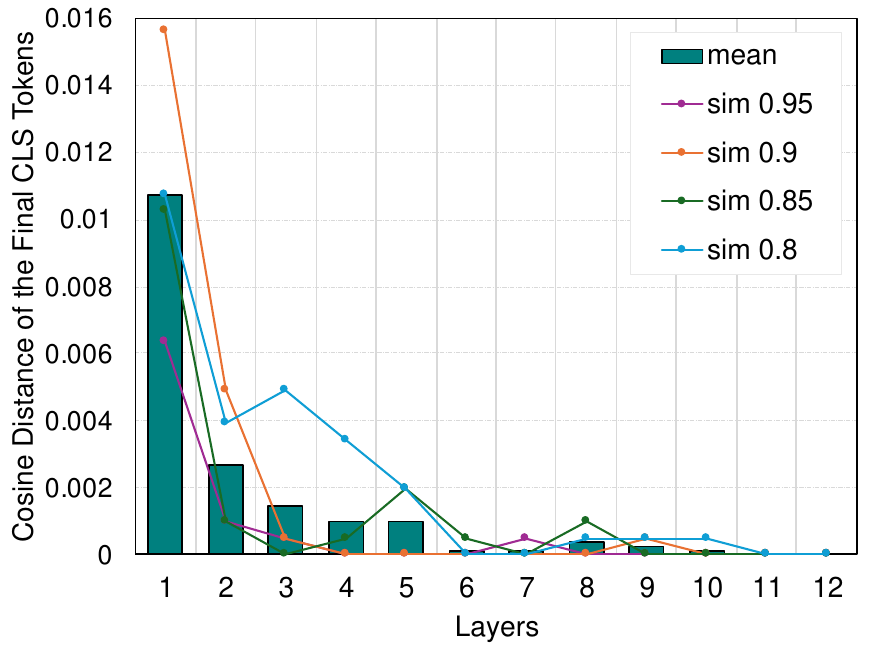} 
    \caption{\textbf{Layer-wise impact of token merging on information loss.} Each line illustrates the amount of information change caused by token merging solely within each layer, measured as the cosine distance of final CLS tokens obtained after merging with the specified similarity threshold (i.e., $0.8$--$0.95$) from their original CLS tokens. The bar represents the average over the four lines.}
    \label{fig2}
\end{figure}


\begin{figure*}[t!]
    \centering
    \includegraphics[width=0.99\textwidth]{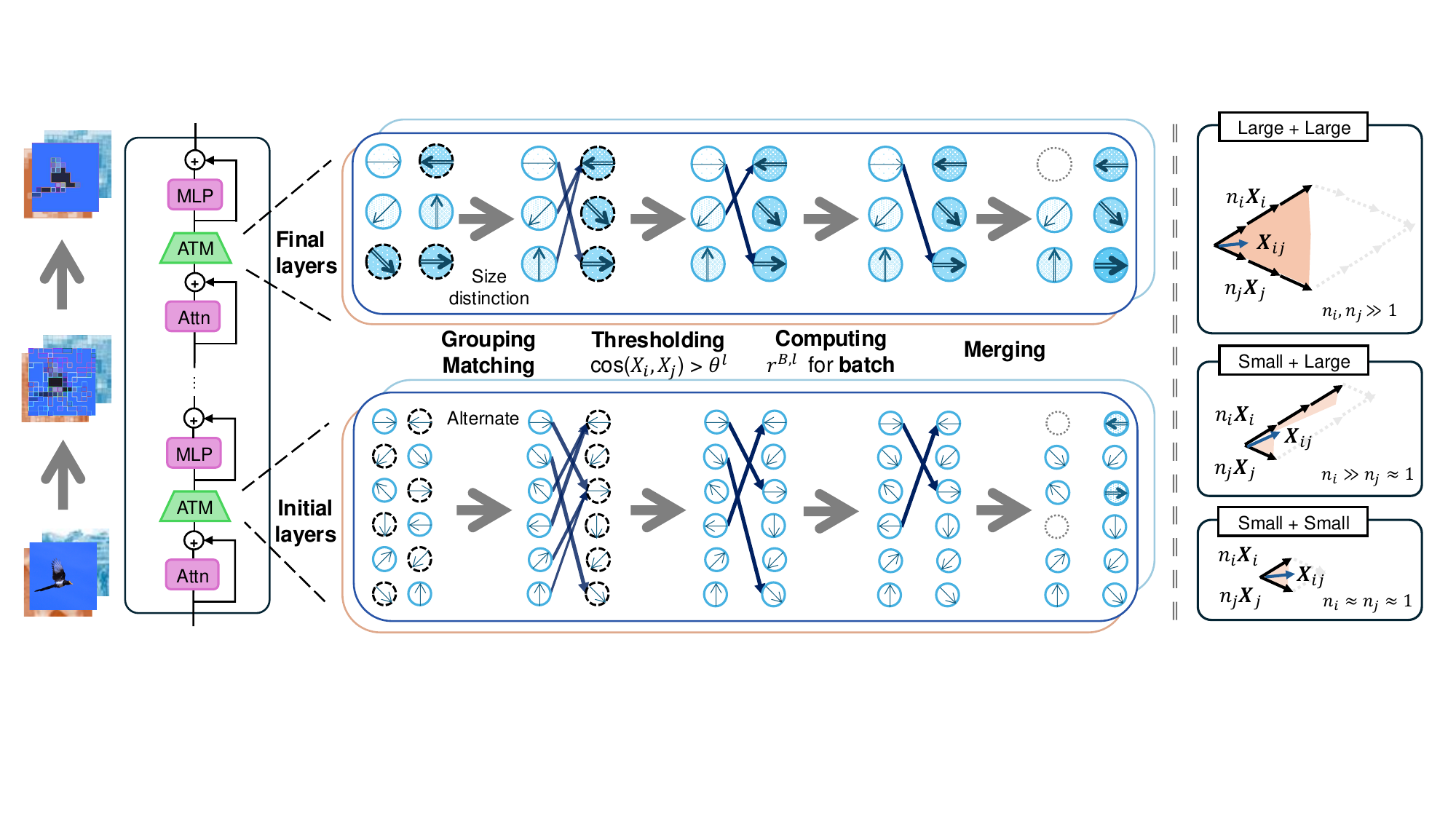} 
    \caption{\textbf{Pipeline of Adaptive Token Merging (ATM).} For each layer, we select a matching strategy, set a layer-dependent threshold, and 
    find the adaptive number $r^{B,l}$ of tokens to reduce for each batch. The token arrows depict their corresponding vector values, and the arrow width and color intensity together reflect the extent of corresponding merging sizes.}
    \label{fig3}
\end{figure*}

\smalltitle{Layer-Wise Impact on Information Loss}
In order to examine the potential information loss caused by merging at each individual layer, we first consider the fact that only the final CLS token is used for actual prediction in DeiT-S. Therefore, we regard the original value of this CLS token as a desired outcome to be generated by compressed versions, and measure the cosine distance between the original CLS token and its changed counterpart after layer-specific merging to quantify the extent of information loss incurred by merging.


Figure \ref{fig2} plots these information losses in the final CLS token when performing layer-specific merging with various similarity thresholds (i.e., 0.8--0.95), where each line corresponds to a particular threshold and the bar represents the average over the lines. It is roughly observed that the information loss increases exponentially as the layer becomes farther away from the final output. This is probably due to the fact that the information loss of merging accumulates more significantly in initial layers than in final ones. Therefore, in the initial layers, token reduction with a threshold can lead to a substantially larger amount of information loss than using the same threshold in the final layers. This suggests that applying a constant threshold across all layers is suboptimal, as the information loss from token merging differs between layers.









\smalltitle{Layer-Dependent Thresholding}
Based on the analysis above, we propose a \textit{layer-dependent thresholding} scheme for properly adjusting the similarity threshold across different layers. Basically, we start with higher thresholds in initial layers and gradually decrease the threshold while approaching final layers. Specifically, inspired by the exponential curve in Figure \ref{fig2}, we design the following decreasing function of the threshold $\theta^l$ at a given layer $l$:
\begin{equation}
    \theta^l = \max\{\alpha - (e^{\beta \cdot (l - 1)} - 1), \theta_{\text{min}}\},
    \label{eq3}
\end{equation}
where $\alpha$ is the initial threshold value, $\beta$ controls the rate of threshold decay, and $\theta_{\text{min}}$ sets the lower bound of the threshold. The lower bound $\theta_{\text{min}}$ intends to prevent the threshold from becoming excessively low in final layers, thereby avoiding erroneous merging of dissimilar tokens in such layers. By adjusting three key parameters, $\alpha$, $\beta$, and $\theta_{\text{min}}$, we can finely control the degree of token merging across different layers, avoiding significant information loss by locally optimizing reduction processes. In our experimental results, this decreasing function turns out to be highly effective in adaptive reduction across layers than constant thresholding.

\subsection{Size-Distinctive Token Matching} \label{sec:size}

Even though the information loss can largely be reduced through layer-dependent thresholding, it is somewhat inevitable to merge less similar tokens in final layers than in initial layers, as the similarity threshold gradually decreases according to our thresholding strategy. Furthermore, remaining tokens in the final layers would have already been merged with multiple tokens along the preceding layers, and therefore each merging operation can result in different amount of information loss even with the same similarity and same layer, depending on the accumulated number of tokens that have been merged, referred to as \textit{merging size}. The merging size of a token is formally defined as follows:
\begin{definition}[Merging Size] \label{def:size}
Given a token $\mathbf{X}_i$, the \textbf{merging size} $n_i$ of $\mathbf{X}_i$ is defined as:
\begin{itemize}
    \item $n_i = 1$ if $\mathbf{X}_i$ is an original token,
    \item $n_i = \sum_j {n_j}$ for all $\mathbf{X}_j$'s with merging size $n_j$ that have been merged to constitute $\mathbf{X}_i$, otherwise.
\end{itemize}
\end{definition}

\smalltitle{Merging Error}
In order to examine how much loss of information can be caused from each merging operation, we define \textit{merging error}, which indicates the amount of information loss incurred from merging two token vectors as follows:
\begin{definition}[Merging Error] \label{def:merging}
   Given two normalized token vectors $\mathbf{X}_i$ and $\mathbf{X}_j$ with merging sizes $n_{i}$ and $n_{j}$, respectively, their merged token $\mathbf{X}_{ij}$ is the weighted average\footnote{The standard approach in token merging \cite{BolyaFDZFH23}} of $\mathbf{X}_i$ and $\mathbf{X}_j$, i.e., $\mathbf{X}_{ij} = \frac{n_{i}\mathbf{X}_i + n_{j}\mathbf{X}_j}{n_{i} + n_{j}}.$
   Then, the \textbf{merging error} $\mathcal{E}_{m}$ is defined as:
   $$\mathcal{E}_{m}(\mathbf{X}_i, \mathbf{X}_j) = n_i ~ \delta(\mathbf{X}_i, \mathbf{X}_{ij}) + n_j ~ \delta(\mathbf{X}_j, \mathbf{X}_{ij}),$$
   where $\delta(\cdot, \cdot) = 1 - \textsc{cos}(\cdot, \cdot)$ represents the cosine distance between two token vectors.
\end{definition}
Note that the merging error does not account for the precise cumulative information loss from all preceding layers. Instead, it focuses on the loss incurred during a single merging operation at a given moment, and hence quantifies total amount of discrepancies between the given tokens and the merged one, treating each token as if it consists of as many copies as its merging size. Based on Definition \ref{def:merging}, we present the following theorem regarding merging error in relation to merging sizes:
\begin{theorem} \label{thm:merging}
$\mathcal{E}_{m}(\mathbf{X}_i, \mathbf{X}_j)$ is $\Theta\left(\frac{n_i n_j} {n_i + n_j}\delta(\mathbf{X}_i, \mathbf{X}_j)\right)$, where $\mathbf{X}_i$ and $\mathbf{X}_j$ are normalized token vectors with merging sizes $n_{i}$ and $n_{j}$, respectively. More specifically, it holds that:
\begin{align}
\frac{n_i n_j}{n_i + n_j} \delta (\mathbf{X}_i , \mathbf{X}_j) \leq \mathcal{E}_{m}(\mathbf{X}_i, \mathbf{X}_j) \leq \frac{2 n_i n_j}{n_i + n_j} \delta (\mathbf{X}_i , \mathbf{X}_j). \label{eq:merging_error}
\end{align}
\end{theorem}
\begin{proof}
    See Appendix \ref{app:theory}.
\end{proof}
According to Theorem \ref{thm:merging}, even when the similarity between tokens being merged is the same, the merging error can vary significantly depending on their merging sizes. In the initial layers, most token pairs being merged are not only highly similar due to a relatively high threshold but also have small merging sizes, leading to negligible merging error, as shown in the `\textit{small + small}' case in Figure \ref{fig3}. In contrast, in the final layers, if both tokens being merged have large merging sizes, the merging error can be substantially higher compared to cases where one or both of the tokens have small merging sizes. To illustrate, suppose either $n_i \gg n_j$ or $n_i \ll n_j$. Then, the resulting merging error would not be considerably high as $\frac{n_i n_j}{n_i + n_j}$ cannot be too large in such cases (the `\textit{small + large}' case in Figure \ref{fig3}). However, when both $n_i$ and $n_j$ are large (the `\textit{large + large}' case in Figure \ref{fig3}), we have $n_i n_j \gg n_i + n_j$, incurring much higher merging error. Let us consider the following specific examples on how much the merging error could be in the three cases in Figure \ref{fig3}. 



\begin{example}[Small + Small]
    Consider the case when $n_i = n_j = 1$, where both $\mathbf{X}_i$ and $\mathbf{X}_j$ are original single tokens that have never been merged. By Eq. (\ref{eq:merging_error}), its merging error will be $c \times \delta(\mathbf{X}_i , \mathbf{X}_j)$, such that $c \in [0.5, 1]$. Thus, the merging error is at most identical to the cosine distance between the two tokens.
\end{example}



\begin{example}[Small + Large]
    If we let $n_i = 100$ and $n_j = 1$ for the case of $n_i \gg n_j$, then we get the merging error of $c \times \delta(\mathbf{X}_i , \mathbf{X}_j)$, such that $c \in [0.99, 1.99]$, obtained from Eq. (\ref{eq:merging_error}). The merging error of this case would be a bit greater than that of the `\textit{small + small}' case, but it is still dominated by the cosine distance alone rather than either of $n_i$ or $n_j$. 
\end{example}

\begin{example}[Large + Large]
    Finally, we consider the case of both tokens having large merging sizes, such as $n_1 = n_2 = 100$. The resulting merging error computed by Eq. (\ref{eq:merging_error}) comes to at least $50 \times \delta(\mathbf{X}_i , \mathbf{X}_j)$ and at most $100 \times \delta(\mathbf{X}_i , \mathbf{X}_j)$, which becomes a few orders of magnitude larger, compared to not only the `small + small' case but also the `small + large' case.
\end{example}

\begin{figure*}[t!]
    \centering
    \subfigure[DeiT-T\label{fig:4a}]{\includegraphics[width=0.32\textwidth]{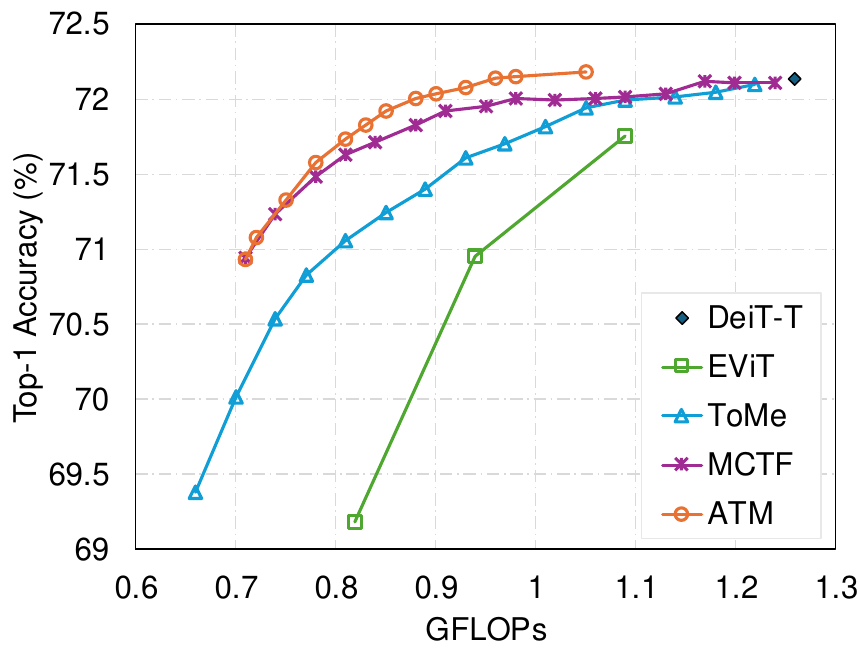}}
    \subfigure[DeiT-S\label{fig:4b}]{\includegraphics[width=0.32\textwidth]{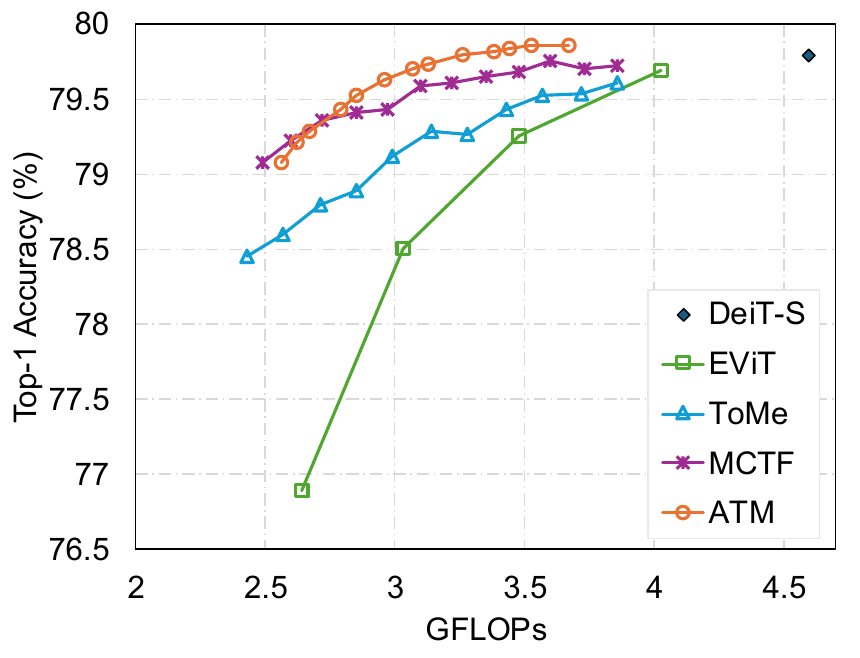}}
    \subfigure[DeiT-B\label{fig:4c}]{\includegraphics[width=0.32\textwidth]{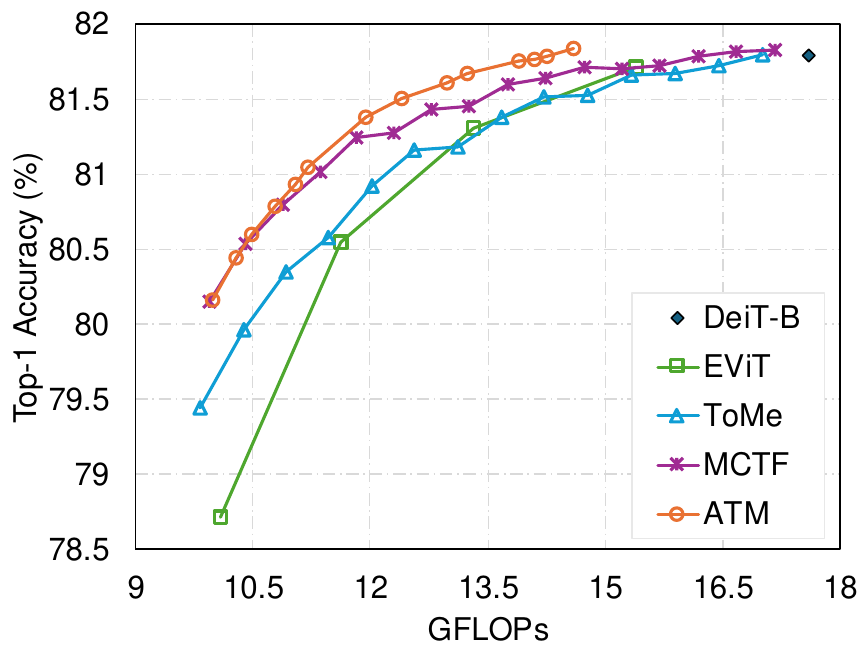}}
    \caption{\textbf{Overall performance comparison of token compression methods on DeiT without extra training.}}
    \label{fig4}
\end{figure*}

\begin{algorithm}[t!]
\footnotesize
    \caption{Layer-Wise Adaptive Token Merging}
    \label{alg}
    \textbf{Input}: Layer depth $l$, Input batch $B$, Batched  input tokens $\mathbf{X}^{l}$\\
    \textbf{Output}: Merged tokens $\mathbf{X}^{l+1}$
    \begin{algorithmic}[1] 

        \STATE $N \gets $ the number of tokens in $\mathbf{X}^{l}$

        \STATE $l_{s} \gets$ the splitting layer depth to select matching strategy

        \IF{$l \in [1,l_{s}]$}
            \STATE $\mathbf{X}_{src}, \mathbf{X}_{dst} \gets  \mathbf{X}^{l}$[:,::2,:], $ \mathbf{X}^ {l}$[:, 1::2,:]
        \ELSE
            \STATE $\mathbf{X}^{l}$ $\gets$ SortByMergingSize($\mathbf{X}^{l}$)
            \STATE $\mathbf{X}_{src}, \mathbf{X}_{dst} \gets \mathbf{X} ^ {l}$[:,:$N$/2,:], $\mathbf{X} ^ {l}$[:, $N$/2:,:]
        \ENDIF

        \STATE $P \gets$ GetTokenPairIndices($\mathbf{X}_{src}$, $\mathbf{X}_{dst}$)
        \STATE $\theta^l \gets \max\{\alpha - (e^{\beta \cdot (l - 1)} - 1), \theta_{\text{min}}\}$
        \STATE $r^{B,l} \gets \frac{1}{|B|} \sum_{b \in B} \sum_{(i,j) \in P} \mathds{1}\left(\textsc{cos}(\mathbf{X}_i^{b,l},~\mathbf{X}_j^{b, l}) > \theta^l\right)$ 
        \STATE $\mathbf{X}^{l+1}$ $\gets$ MergeTokens($\mathbf{X}^{l}$, $P$, $r^{B,l}$)
        \RETURN $\mathbf{X}^{l+1}$

    \end{algorithmic}
\end{algorithm}

\smalltitle{Size-Distinctive Token Matching}
Based on the observation in Theorem \ref{thm:merging}, we propose \textit{size-distinctive token matching}, which replaces the typical alternate grouping with size-distinctive grouping in bipartite token matching for the final layers. More specifically, as illustrated in Figure \ref{fig3}, we first sort the tokens by their merging sizes and then split the sorted group in half to form two groups, one containing tokens with small merging sizes and the other containing with large merging sizes. This size-distinctive grouping ensures that token pairs with large merging sizes are not merged together, thereby minimizing the information loss potentially caused by high merging error in the final layers.

\subsection{Batch-Adaptive Merging}
Although our threshold-based merging would nicely work with individual images, varying the number of tokens across images disables batched inference, which has justified existing static approaches \cite{BolyaFDZFH23,KimGHSJ24} in practice. In the proposed ATM method, we address this limitation by supporting batch-adaptive merging, similar to one in token pruning \cite{LiuWG23}. By batch-adaptive merging, we can still take advantage of adaptive token reduction while achieving practical efficiency of processing multiple images simultaneously. More specifically, for each batch $B$, we adaptively define the following number $r^{B, l}$ of tokens to reduce for each layer $l$:
\begin{align}
    r^{B, l} = &\frac{1}{|B|} \sum_{b \in B} \sum_{(i,j) \in P} \mathds{1}\left(\textsc{cos}(\mathbf{X}_i^{b, l},~\mathbf{X}_j^{b, l}) > \theta^l\right), \nonumber
\end{align}
where (i) $\textsc{cos}(\mathbf{X}_i^{b, l},~\mathbf{X}_j^{b, l})$ is the cosine similarity between a pair of token pairs, $\mathbf{X}_i^{b,l}$ and $\mathbf{X}_j^{b,l}$, in layer $l$ for each image $b$ in the batch, (ii) $P$ is the set of pairs of token indices $(i, j)$ to be compared between two groups, and (iii) $\theta^l$ is the layer-dependent threshold determined by Eq. (\ref{eq3}). Thus, depending on how many tokens are similar enough on the average over the batch, we dynamically determine $r^{B, l}$, which also varies across layers by utilizing our layer-dependent thresholding.

\smalltitle{Overall Process}
In summary, our entire ATM process is conceptually shown in Figure \ref{fig3} and its detailed steps are presented in Algorithm \ref{alg}. In this algorithm, we first select a different matching strategy based on the current layer depth (either alternate matching in Lines 3--4 or size-distinctive matching in Lines 5--7), and then identify token pairs with the highest similarities using layer-dependent thresholding (Line 10). Subsequently, we determine the appropriate number of token pairs to merge for batched inference (Line 11), and finally perform the token merging process (Line 12). Through these comprehensive techniques, we can merge highly similar token pairs even with batched processing while minimizing information loss during the merging process.

\section{Experiments}
In this section, we present a detailed performance study of ATM with a focus on its effectiveness in a training-free scenario, while also comparing it to training-required methods, highlighting its competitive performance without fine-tuning.



\subsection{Experimental Setup}
\smalltitle{Dataset, Models, and Metrics}
To validate ATM's performance, we employ the widely-used ImageNet-1K classification dataset \cite{DengDSLL009} for comparative experiments using various pretrained models, including DeiT \cite{TouvronCDMSJ21}, LV-ViT \cite{JiangHYZSJWF21}, and T2T-ViT \cite{0007CWYSJTFY21} families, as well as a self-supervised pretrained model (MAE), to demonstrate the broad applicability of our method. For performance assessment, we measure the three crucial metrics: top-1 accuracy, floating-point operations (FLOPs), and training time.

\smalltitle{Implementation Details}
For all pretrained models but T2T-ViT\footnote{Due to insufficient GPU memory, we use 100 for T2T-ViT, according to its original default batch size.}, we consistently use a batch size of 1024 by default. To ensure comprehensive coverage of state-of-the-art results, we leverage official implementations whenever available and also include partial results obtained from reference papers and their codebases. It is worth emphasizing that the tables and figures in this paper roughly form a \textit{superset} of previously reported results from the existing works. All experiments are conducted on a server equipped with an NVIDIA RTX A6000 GPU. See Appendix \ref{app:implement} for further implementation details.





\subsection{Performance in Training-Free Scenario}
We first evaluate the performance of ATM without any additional training, compared to state-of-the-art training-free token compression methods, including EViT \cite{LiangGTS0X22}, ToMe \cite{BolyaFDZFH23}, MCTF \cite{Leeck24}, and Zero-TP \cite{WangDj24}. For this performance test, we mainly use \textit{off-the-shelf} DeiT family, as it is most widely adopted in the experiments of token compression.

\begin{table}[t!]
\caption{\textbf{Stress test of token compression without allowing accuracy drop of DeiT without extra training.}}
\centering
\resizebox{0.9\columnwidth}{!}{%
\begin{tabular}{ccccc}
\Xhline{2.0\arrayrulewidth}
Model                    & Method                                & \begin{tabular}[c]{@{}c@{}}Params\\ (M)\end{tabular} & \begin{tabular}[c]{@{}c@{}}Top-1 Acc.\\ (\%)\end{tabular} & GFLOPs                                        \\ \hline
                         & Baseline                              & 5                                                    & 72.1                                                      & 1.3                                           \\ 
                         & EViT                                 & 5                                                    & 72.1                                                      & 1.2 (-7\%)                                    \\
                         & ToMe                                 & 5                                                    & 72.1                                                      & 1.2 (-7\%)                                    \\
                         & MCTF                                 & 5                                                    & 72.1                                                      & 1.2 (-7\%)                                    \\
\multirow{-5}{*}{DeiT-T} & \cellcolor[HTML]{EFEFEF}\textbf{ATM} & \cellcolor[HTML]{EFEFEF}\textbf{5}                   & \cellcolor[HTML]{EFEFEF}\textbf{72.1}                     & \cellcolor[HTML]{EFEFEF}\textbf{0.9 (-31\%)}  \\ \hline
                         & Baseline                              & 22                                                   & 79.8                                                      & 4.6                                           \\
                         & EViT                                 & 22                                                   & 79.8                                                      & 4.2 (-9\%)                                    \\
                         & ToMe                                 & 22                                                   & 79.8                                                      & 4.2 (-9\%)                                    \\
                         & MCTF                                 & 22                                                   & 79.8                                                      & 3.6 (-22\%)                                   \\
                         & Zero-TP                              & 22                                                   & 79.8                                                      & 4.0 (-13\%)                                   \\
\multirow{-6}{*}{DeiT-S} & \cellcolor[HTML]{EFEFEF}\textbf{ATM} & \cellcolor[HTML]{EFEFEF}\textbf{22}                  & \cellcolor[HTML]{EFEFEF}\textbf{79.8}                     & \cellcolor[HTML]{EFEFEF}\textbf{3.2 (-30\%)}  \\ \hline
                         & Baseline                              & 86                                                   & 81.8                                                      & 17.6                                          \\
                         & EViT                                 & 86                                                   & 81.8                                                      & 15.9 (-10\%)                                  \\
                         & ToMe                                 & 86                                                   & 81.8                                                      & 17.0 (-3\%)                                   \\
                         & MCTF                                 & 86                                                   & 81.8                                                      & 16.2 (-8\%)                                   \\
\multirow{-5}{*}{DeiT-B} & \cellcolor[HTML]{EFEFEF}\textbf{ATM} & \cellcolor[HTML]{EFEFEF}\textbf{86}                  & \cellcolor[HTML]{EFEFEF}\textbf{81.8}                     & \cellcolor[HTML]{EFEFEF}\textbf{13.9 (-21\%)} \\ \Xhline{2.0\arrayrulewidth}
\end{tabular}%
}
\label{tab1}
\end{table}


\smalltitle{Overall Performance}
In Figure \ref{fig4}, we present the overall performance of compared methods, showing how the resulting accuracy changes as we vary the target computational cost (i.e., FLOPs). It is observed that ATM consistently outperforms the other methods throughout all three graphs, positioning in the top-left corner of the graphs, indicating higher accuracy and lower computational cost.

\smalltitle{Stress Testing}
In various practical scenarios, such as safety-critical systems, it is often demanding not to compromise model's original performance while improving the computational efficiency. In Table \ref{tab1}, we conduct a \textit{stress test} on DeiT to assess the maximum token compression achievable by each method without any loss in original accuracy of baseline. In this stress test, ATM reaches the highest compression level with no drop in accuracy, clearly setting a new benchmark, compared to the existing methods. Notably, for the DeiT-T and DeiT-S models, ATM endures over 30\% reduction in FLOPs without any performance degradation. These results confirm that ATM can effectively preserve the representational power of the baseline model, while performing substantial reduction in computational costs.

\smalltitle{Comparison on Various Backbones}
Beyond DeiT, we further conduct performance comparison using various architectures and self-supervised trained models. For a fair comparison, we adjust FLOPs of ATM to be less than or equal to those of referenced works, and report the resulting accuracy. As shown in Table \ref{tab2}, our ATM method always shows the best performance in every backbone model. Particularly, for T2T-ViT$_\text{t}$-19 model, we observe the largest accuracy gap between ToMe \cite{BolyaFDZFH23} and ATM. Considering that ToMe is a static token merging method and T2T-ViT has a deep-narrow structure, where information loss could accumulate substantially, this outcome proves the effectiveness of our layer-dependent thresholding approach.

\begin{table}[t!]
\caption{\textbf{Performance comparison on various ViT backbones without extra training.}}
\centering
\resizebox{0.9\columnwidth}{!}{%
\begin{tabular}{ccccc}
\Xhline{2.0\arrayrulewidth}
Model                           & Method                               & GFLOPs                                & \begin{tabular}[c]{@{}c@{}}Params\\ (M)\end{tabular} & \begin{tabular}[c]{@{}c@{}}Top-1 Acc.\\ (\%)\end{tabular} \\ \hline
                                & Baseline                             & 1.3                                   & 5                                                    & 72.1                                                      \\
                                & EViT                                 & 0.9                                   & 5                                                    & 71.0 (-1.1)                                                     \\
                                & ToMe                                 & 0.9                                   & 5                                                    & 71.6 (-0.5)                                                     \\
                                & MCTF                                 & 0.9                                   & 5                                                    & 71.9 (-0.2)                                                     \\
\multirow{-5}{*}{DeiT-T}        & \cellcolor[HTML]{EFEFEF}\textbf{ATM} & \cellcolor[HTML]{EFEFEF}\textbf{0.9}  & \cellcolor[HTML]{EFEFEF}\textbf{5}                   & \cellcolor[HTML]{EFEFEF}\textbf{72.1 (-)}                     \\ \hline
                                & Baseline                             & 4.6                                   & 22                                                   & 79.8                                                      \\
                                & EViT                                 & 3.0                                   & 22                                                   & 78.5 (-1.3)                                                     \\
                                & ToMe                                 & 3.0                                   & 22                                                   & 79.1 (-0.7)                                                     \\
                                & MCTF                                 & 3.0                                   & 22                                                   & 79.5 (-0.3)                                                     \\
                                & Zero-TP                              & 3.0                                   & 22                                                   & 79.4 (-0.4)                                                     \\
\multirow{-6}{*}{DeiT-S}        & \cellcolor[HTML]{EFEFEF}\textbf{ATM} & \cellcolor[HTML]{EFEFEF}\textbf{3.0}  & \cellcolor[HTML]{EFEFEF}\textbf{22}                  & \cellcolor[HTML]{EFEFEF}\textbf{79.7 (-0.1)}                     \\ \hline
                                & Baseline                             & 2.9                                   & 8.5                                                  & 79.1                                                      \\
                                & ToMe                                 & 2.0                                   & 8.5                                                  & 78.4 (-0.7)                                                     \\
\multirow{-3}{*}{LV-ViT-T}      & \cellcolor[HTML]{EFEFEF}\textbf{ATM} & \cellcolor[HTML]{EFEFEF}\textbf{2.0}  & \cellcolor[HTML]{EFEFEF}\textbf{8.5}                 & \cellcolor[HTML]{EFEFEF}\textbf{78.6 (-0.5)}                     \\ \hline
                                & Baseline                             & 6.6                                   & 26.2                                                 & 83.3                                                      \\
                                & EViT                                 & 3.9                                   & 26.2                                                 & 79.8 (-3.5)                                                     \\
                                & ToMe                                 & 3.6                                   & 26.2                                                 & 81.3 (-2.0)                                                     \\
                                & Zero-TP                              & 3.5                                   & 26.2                                                 & 81.5 (-1.8)                                                     \\
\multirow{-5}{*}{LV-ViT-S}      & \cellcolor[HTML]{EFEFEF}\textbf{ATM} & \cellcolor[HTML]{EFEFEF}\textbf{3.5}  & \cellcolor[HTML]{EFEFEF}\textbf{26.2}                & \cellcolor[HTML]{EFEFEF}\textbf{82.1 (-1.2)}                      \\ \hline
                                & Baseline                             & 6.1                                   & 21.5                                                 & 81.7                                                      \\
                                & ToMe                                 & 4.1                                   & 21.5                                                 & 81.1 (-0.6)                                                     \\
\multirow{-3}{*}{T2T-ViT$_\text{t}$-14} & \cellcolor[HTML]{EFEFEF}\textbf{ATM} & \cellcolor[HTML]{EFEFEF}\textbf{4.1}  & \cellcolor[HTML]{EFEFEF}\textbf{21.5}                & \cellcolor[HTML]{EFEFEF}\textbf{81.4 (-0.3)}                     \\ \hline
                                & \cellcolor[HTML]{FFFFFF}Baseline     & \cellcolor[HTML]{FFFFFF}9.8           & \cellcolor[HTML]{FFFFFF}39.2                         & \cellcolor[HTML]{FFFFFF}82.4                              \\
                                & ToMe                                 & 4.8                                   & 39.2                                                 & 78.9 (-3.5)                                                     \\
\multirow{-3}{*}{T2T-ViT$_\text{t}$-19} & \cellcolor[HTML]{EFEFEF}\textbf{ATM} & \cellcolor[HTML]{EFEFEF}\textbf{4.8}  & \cellcolor[HTML]{EFEFEF}\textbf{39.2}                & \cellcolor[HTML]{EFEFEF}\textbf{81.2 (-1.2)}                     \\ \hline
                                & \cellcolor[HTML]{FFFFFF}Baseline     & \cellcolor[HTML]{FFFFFF}17.6          & \cellcolor[HTML]{FFFFFF}86                           & \cellcolor[HTML]{FFFFFF}83.7                              \\
                                & EViT                                 & 11.5                                  & 86                                                   & 82.0 (-1.7)                                                     \\
                                & ToMe                                 & 11.5                                  & 86                                                   & 82.3 (-1.4)                                                     \\
\multirow{-4}{*}{MAE}           & \cellcolor[HTML]{EFEFEF}\textbf{ATM} & \cellcolor[HTML]{EFEFEF}\textbf{11.5} & \cellcolor[HTML]{EFEFEF}\textbf{86}                  & \cellcolor[HTML]{EFEFEF}\textbf{82.5 (-1.2)}                     \\ \Xhline{2.0\arrayrulewidth}
\end{tabular}%
}
\label{tab2}
\end{table}


\subsection{Ablation Studies}

\smalltitle{Effectiveness of Layer-Dependent Thresholding}
To examine the effectiveness of our layer-dependent thresholding, we compare our merging scheme to possible alternative merging schedules presented in Figure \ref{fig5}, including both top-\textit{r} and threshold-based variants. For this analysis, we equalize the matching procedure for all the compared approaches to be alternate matching. Unlike constant top-$r$, which represents a pure static approach of constantly merging $r$ token pairs as in ToMe \cite{BolyaFDZFH23}, increasing and decreasing top-$r$ approaches linearly increment and decrement $r$ by one per layer, respectively. Constant thresholding is threshold-based merging yet with a fixed threshold across layers. By properly adjusting threshold (not just incrementing or decrementing $r$ itself), our layer-dependent thresholding manages to achieve the best performance for all settings in Figure \ref{fig5}, followed by the increasing top-$r$ approach. These results further validate our analysis in Figure \ref{fig2}, and well support our claim that more similar tokens should be merged in initial layers to reduce the accumulated loss. Interestingly, the increasing top-$r$ approach naturally follows this desired merging scheme, which explains its second position in the graph. Meanwhile, constant thresholding shows unexpectedly low performance, even though it is also a type of adaptive merging using a constant threshold. 

To reveal the underlying reason of this large gap between constant and layer-dependent thresholding, we presents the actual numbers of merged tokens across layers for the same target computational cost (33\% reduction in FLOPs) in Figure \ref{fig6}. Each measurement is the average over all the images in the ImageNet-1K dataset. It turns out that constant thresholding excessively merges tokens in the first and last layers, not only causing information loss in the initial layers but also leading to inefficient merging that heavily relies on the reduction at the end of the inference pass. On the other hand, our layer-dependent thresholding approach significantly reduces the merging of dissimilar token pairs in the first layer and inefficient merging in the last layers by properly decrementing layer-wise thresholds. This enables ATM to reduce information loss and to enhance the merging efficiency.



\begin{table*}[t!]
\caption{\textbf{Sensitivity analysis of the splitting layer depth.} $l_s$ denotes the last layer depth where alternative matching is applied. Thus, `$l_s = 0$' or `$l_s = 12$' corresponds to applying size-distinctive matching or alternative matching across all layers, respectively.}
\centering
\resizebox{0.9\textwidth}{!}{%
\begin{tabular}{cccccccccc
>{\columncolor[HTML]{EFEFEF}}c ccc}
\Xhline{2.0\arrayrulewidth}
$l_s$           & 0     & 1     & 2     & 3     & 4     & 5     & 6     & 7     & 8     & 9     & 10    & 11    & 12    \\ \hline
Top-1 Acc. (\%) & 79.68 & 79.74 & 79.68 & 79.69 & 79.70 & 79.69 & 79.64 & 79.66 & 79.70 & 79.76 & 79.67 & 79.65 & 79.65 \\
GFLOPs          & 3.95  & 3.38  & 3.32  & 3.30  & 3.33  & 3.31  & 3.30  & 3.27  & 3.25  & 3.23  & 3.22  & 3.21  & 3.21  \\ \Xhline{2.0\arrayrulewidth}
\end{tabular}%
}
\label{tab:A1}
\end{table*}

\begin{figure}[t!]
    \centering
    \includegraphics[width=0.84\columnwidth]{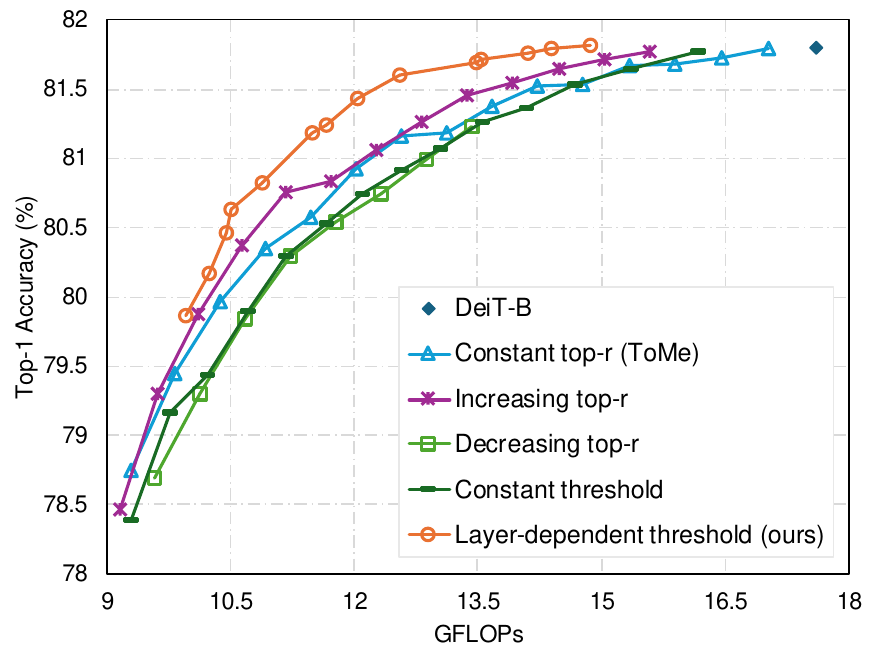}
    \caption{\textbf{Performance comparison on merging schedules.}}
    \label{fig5}
\end{figure}

\begin{figure}[t!]
    \centering
    \subfigure[Constant thresholding\label{fig:6a}]{\includegraphics[width=0.48\columnwidth]{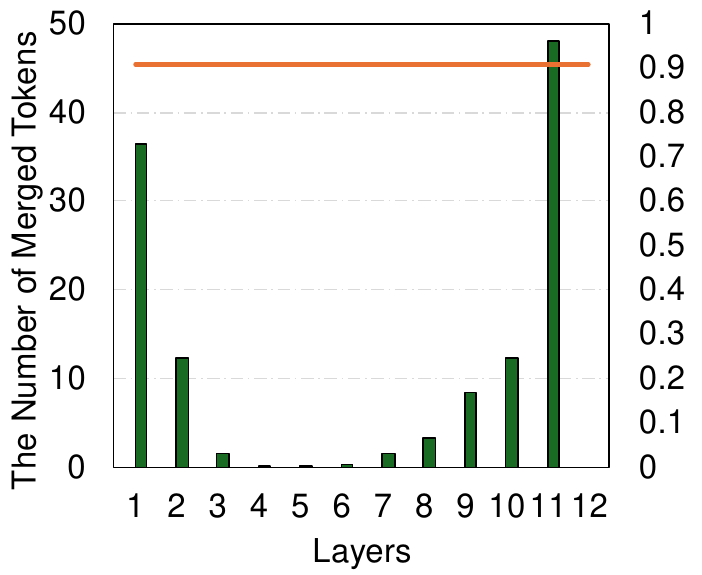}}
    \subfigure[Layer-dependent thresholding\label{fig:6b}]{\includegraphics[width=0.48\columnwidth]{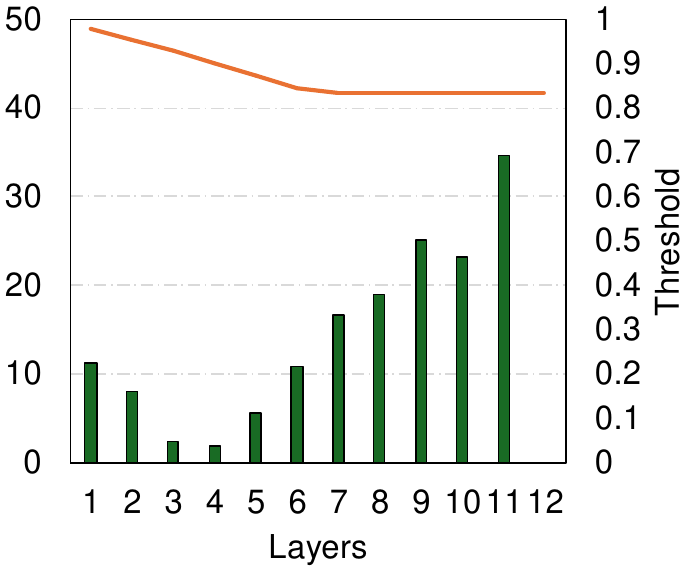}}
    \caption{\textbf{Distribution of the numbers of merged tokens across layers.} Each orange line illustrates the threshold value along the layers, and each green bar represents the number of merged tokens for the corresponding layer.} 
    \label{fig6}
\end{figure}

\begin{figure*}[t!]
    \centering
    \includegraphics[width=0.93\textwidth]{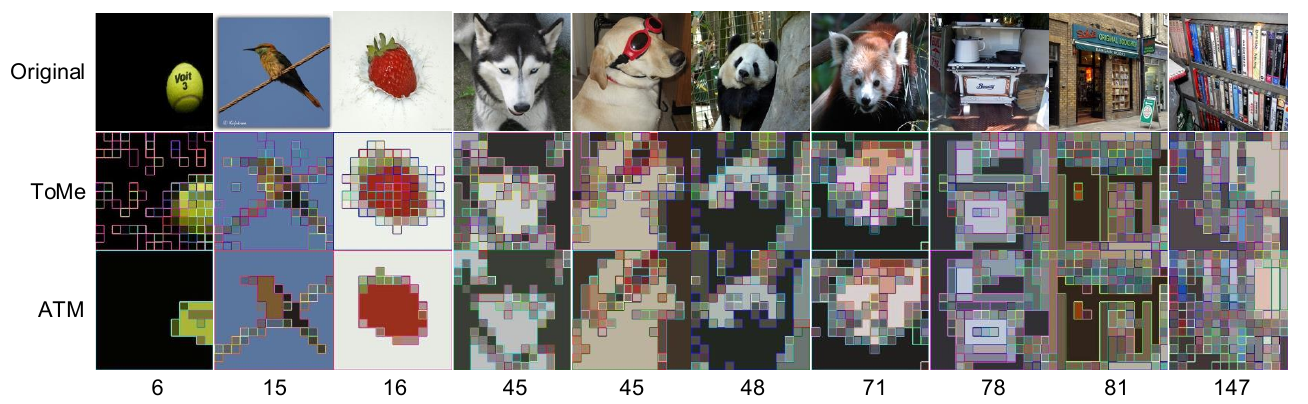} 
    \caption{\textbf{Visualizations.} Each number indicates the number of finally remaining tokens in ATM, while equally 77 tokens are left in all images resulting from ToMe.}
    \label{fig7}
\end{figure*}

\smalltitle{Alternate Matching vs. Size-Distinctive Matching}
Another important consideration of ATM is how to adjust the splitting layer depth, which defines the boundary between applying alternate and size-distinctive matching. Table \ref{tab:A1} presents the performance results for various splitting layer depths, where $l_s \in [0, 12]$ denotes the last layer of alternate matching. Note that due to ATM's adaptive behavior, its resulting computational cost (i.e., FLOPs) can differ with different $l_s$ settings. It is observed that the best performance can be obtained when size-distinctive matching is applied to roughly the final quarter of the total layer depth, whereas the performance decreases when size-distinctive matching is not used, as seen in the case of $l_s = 12$. Meanwhile, starting size-distinctive matching too early, such as at $l_s = 0$, also results in reduced performance as well as relatively higher FLOPs, which confirms that merging sizes have less impact in the initial layers, as mentioned in Section \ref{sec:size}. Based on these observations, we set the splitting layer depth at approximately three-quarters of the total layer depth in all our experiments.

\subsection{Visualizations}
To qualitatively analyze the performance of ATM, particularly compared to static top-$r$ merging (i.e., ToMe \cite{BolyaFDZFH23}), Figure \ref{fig7} presents visualizations using different images of various complexities, arranging from simple to complex (left to right). Each resulting compressed image is obtained from the final layer of DeiT-S after applying either ToMe or ATM with the same reduction rate in FLOPs. As clearly observed, static merging applies a uniform merging quantity regardless of the complexity of images, showing insufficient merging for a simple sample yet excessive merging for a complex one. For example, the left-most image has a large area of background, in which however many tokens remain not merged in ToMe. In contrast, our ATM method adaptively controls the level of merging quantity based on the complexity of images. More comprehensive visualizations are provided in the Appendix \ref{app:vis}.


\subsection{Performance Comparison with Extra Training}
Lastly, we provide performance comparison with other state-of-the-art methods that fully utilize extra training for performance recovery. As shown in Table \ref{tab3}, many of these methods require heavy additional training (e.g., 300 epochs) to restore their degraded performance, thereby achieving their final accuracies. Remarkably, our training-free version of ATM already surpasses most of these computationally-intensive methods, and even with FLOPs reduced to 2.6, achieves zero accuracy drop with only 30 epochs of fine-tuning. Interestingly, MCTF with additional training reports an accuracy even higher than that of the original model. This is attributed to their distillation-like strategy, requiring an additional teacher model and training process, thereby increasing the overall training time substantially. In contrast, we simply follow the general training procedure of the baseline model \cite{TouvronCDMSJ21}, and thanks to our minimal damage during token compression, ATM enables fast training time while still achieving baseline accuracy.

\begin{table}[t!]
\caption{\textbf{Performance comparison on DeiT-S with additional training.} `$\dagger$' indicates training-free versions.}
\resizebox{\columnwidth}{!}{%
\begin{tabular}{ccccc}
\Xhline{2.0\arrayrulewidth}
Method                      & GFLOPs                      & 
\begin{tabular}[c]{@{}c@{}}Top-1 Acc.\\ (\%)\end{tabular} & Epochs
& \begin{tabular}[c]{@{}c@{}}Training Time\\ (hours)\end{tabular}
\\ \hline
                          DeiT-S (Baseline)                    & 4.6                         & 79.8                                                      & -        & -                          \\
                          DynamicViT \cite{RaoZLLZH21}                 & 2.9                         & 79.3 (-0.5)                                               & 30        & 43                         \\
                           IA-RED$^2$ \cite{PanPJWFO21}                     & 3.2                         & 79.1 (-0.7)                                               & 90       & -                          \\
                          Evo-ViT \cite{XuZZSLDZXS22}                      & 3.0                         & 79.4 (-0.4)                                               & 300         & 126                      \\
                          EViT \cite{LiangGTS0X22}                        & 3.0                         & 79.5 (-0.3)                                               & 300        & 95                       \\
                          A-ViT \cite{YinVAMKM22}                       & 3.6                         & 78.6 (-1.2)                                               & 100           & 65                   \\
                          ToMe \cite{BolyaFDZFH23}                      & 2.7                         & 79.4 (-0.4)                                               & 300             & 87                 \\
                          BAT \cite{LongZP0023}                         & 3.0                         & 79.6 (-0.2)                                               & 300            & -                  \\
                          MCTF \cite{Leeck24}                        & \textbf{2.6}                & \textbf{80.1 (+0.3)}                                      & 30              & 39                 \\
                           \cellcolor[HTML]{EFEFEF}ATM$^\dagger$ & \cellcolor[HTML]{EFEFEF}3.0 & \cellcolor[HTML]{EFEFEF}79.7 (-0.1)                       & \cellcolor[HTML]{EFEFEF}\textbf{0} 
                                        & \cellcolor[HTML]{EFEFEF}- \\
 \cellcolor[HTML]{EFEFEF}ATM & \cellcolor[HTML]{EFEFEF}\textbf{2.6} & \cellcolor[HTML]{EFEFEF}79.8 (-)                          & \cellcolor[HTML]{EFEFEF}30   
                                            & \cellcolor[HTML]{EFEFEF} \textbf{11} \\ \Xhline{2.0\arrayrulewidth}
\end{tabular}%
}
\label{tab3}
\end{table}

\section{Conclusion}
In this work, we introduced Adaptive Token Merging (ATM), a novel method designed for lossless reduction of redundant tokens without requiring additional training.
ATM utilizes a threshold-based merging approach that adapts to the characteristics of individual images, while dynamically adjusting thresholds across layers. Moreover, based on our theoretical insights on merging error, ATM applies a novel size-distinctive token matching strategy in the final layers to avoid undesirable merging of tokens with large merging sizes. Building upon this, ATM supports batch-adaptive merging, addressing the practical need for batch processing while maintaining adaptive token reduction. Through extensive experiments, we demonstrated that ATM delivers exceptional performance, both quantitatively and qualitatively, even outperforming most of the training-intensive methods without the need for extra training.

\begin{ack}
This work was supported in part by Institute of Information \& communications Technology Planning \& Evaluation
(IITP) grants funded by the Korea government(MSIT) (No.2022-0-00448, Deep Total Recall: Continual Learning
for Human-Like Recall of Artificial Neural Networks, No.RS-2022-00155915, Artificial Intelligence Convergence
Innovation Human Resources Development (Inha University)), and in part by INHA UNIVERSITY Research Grant.
\end{ack}

\bibliography{mybibfile}

\clearpage

\setcounter{table}{0}
\setcounter{figure}{0}

\renewcommand{\thetable}{A\arabic{table}}
\renewcommand{\thefigure}{A\arabic{figure}}

\appendix

\onecolumn
\section*{\centering Technical Appendix}
\section*{\centering Lossless Token Merging Even Without Fine-Tuning in Vision Transformers}

In this appendix, we first present full implementation details of ATM and our theoretical analysis on merging error. We then provide comprehensive experimental results, including sensitivity analysis of batch sizes, full results with varying FLOPs and more visualizations.


\section{Implementation Details} \label{app:implement}
As illustrated in the main paper, we conduct experiments with various model families including DeiT, LV-ViT, T2T-ViT, and MAE. We merge tokens based on the similarity of key pairs for all models. Among these models, we implement proportional attention except for MAE, which tracks the merging sizes of tokens and incorporates them into the attention matrix, following existing token merging \cite{BolyaFDZFH23,Leeck24}. The proportional attention is formulated as:
$$\mathbf{A} = \text{softmax}\left(\frac{\mathbf{Q}\mathbf{K}^T}{\sqrt{d}} + \log \mathbf{s}\right),
$$
where (i) $\mathbf{A}$ denotes the attention matrix, (ii) $\mathbf{s}$ represents the matrix of token merging sizes, (iii) $d$ is the dimension of the key vectors, and (iv) $\mathbf{Q}$ and $\mathbf{K}$ are query and key vectors, respectively.
Furthermore, since most architectures (e.g., DeiT, LV-ViT, T2T-ViT) solely use a CLS token for image classification in the MLP head after passing through all ViT blocks, we retain only the CLS token and remove all other image tokens in the last layer, instead of merging them, to further reduce computational costs, except for MAE.

For fine-tuning, we use a batch size of 256, 512, or 1024 depending on the available GPU memory with a single NVIDIA RTX A6000 GPU, set the number of warm-up epochs to 0, and set the total number of epochs to 30. We employ a cosine decay schedule for adjusting the learning rate, setting the initial learning rate to 2e-5 and the final learning rate to 3e-6. All other settings follow the baseline model's configuration.

\section{Theoretical Analysis on Merging Error} \label{app:theory}
\subsection{Proof of Theorem \ref{thm:merging}} \label{app:proof}

\smalltitle{Norm of merged tokens}
Considering the weighted averaging formula in Definition \ref{def:merging}, the norm of a merged token $\mathbf{X}_{ij}$ is interpreted as follows:
\begin{align}
    \|\mathbf{X}_{ij}\| 
    &= \sqrt{\frac{(n_i\mathbf{X}_i + n_j\mathbf{X}_j)\cdot(n_i\mathbf{X}_i + n_j\mathbf{X}_j)}{(n_i + n_j)^2}} \nonumber \\
    &= \sqrt{\frac{n_i^2\mathbf{X}_i \cdot \mathbf{X}_i + 2n_in_j\mathbf{X}_i \cdot \mathbf{X}_j + n_j^2\mathbf{X}_j \cdot \mathbf{X}_j}{(n_i + n_j)^2}} \label{eq:norm_frac}\\
    &= \sqrt{\frac{n_i\mathbf{X}_i \cdot (n_i\mathbf{X}_i + n_j\mathbf{X}_j) + n_j\mathbf{X}_j \cdot (n_i\mathbf{X}_i + n_j\mathbf{X}_j)}{(n_i + n_j)^2}} \nonumber \\
    &= \sqrt{\frac{n_i\mathbf{X}_i \cdot \mathbf{X}_{ij} + n_j\mathbf{X}_j \cdot \mathbf{X}_{ij}}{n_i + n_j}}.  \nonumber
\end{align}
Therefore, we have:
\begin{align}
     n_i\mathbf{X}_i \cdot \mathbf{X}_{ij} + n_j\mathbf{X}_j \cdot \mathbf{X}_{ij} = (n_i + n_j)\|\mathbf{X}_{ij}\|^2. \label{eq:norm}
\end{align}

\smalltitle{Proof of Theorem \ref{thm:merging}}
By using this interpretation of $\|\mathbf{X}_{ij}\|$, we can derive $\mathcal{E}_{m}(\mathbf{X}_i, \mathbf{X}_j)$ as follows:
\begin{align}
\mathcal{E}_{m}(\mathbf{X}_i, \mathbf{X}_j) 
&= n_i \delta(\mathbf{X}_i, \mathbf{X}_{ij}) + n_j \delta(\mathbf{X}_j, \mathbf{X}_{ij}) \nonumber \\
\shortintertext{Since both $\mathbf{X}_i$ and $\mathbf{X}_j$ are normalized (i.e., unit vectors):}
&= n_i (1 - \frac{\mathbf{X}_i \cdot \mathbf{X}_{ij}}{\|\mathbf{X}_{ij}\|}) + n_j (1 - \frac{\mathbf{X}_j \cdot \mathbf{X}_{ij}}{\|\mathbf{X}_{ij}\|}) \nonumber \\
&= n_i + n_j - \frac{n_i\mathbf{X}_i \cdot \mathbf{X}_{ij} + n_j\mathbf{X}_j \cdot \mathbf{X}_{ij}}{\|\mathbf{X}_{ij}\|} \nonumber \\
\shortintertext{By Eq. (\ref{eq:norm}):}
&= n_i + n_j - (n_i + n_j) \|\mathbf{X}_{ij}\| \nonumber \\
\shortintertext{By Eq. (\ref{eq:norm_frac}):}
&= n_i + n_j - \sqrt{n_i^2 + n_j^2 + 2n_in_j\mathbf{X}_i \cdot \mathbf{X}_j} \nonumber \\
\shortintertext{For brevity, let $N = n_i + n_j$ and $M = n_i n_j$. Then:}
 = &~~\sqrt{N^2}  - \sqrt{N^2 - 2 M + 2 M\mathbf{X}_i \cdot \mathbf{X}_j } \nonumber \\
= &~~\sqrt{N^2} - \sqrt{N^2 - 2 M \delta (\mathbf{X}_i , \mathbf{X}_j) } \nonumber \\
= &~~\frac{N^2 - (N^2 - 2 M\delta (\mathbf{X}_i , \mathbf{X}_j))}{\sqrt{N^2} + \sqrt{N^2 - 2 M \delta (\mathbf{X}_i , \mathbf{X}_j)}} \nonumber \\
= &~~\frac{2 M \delta (\mathbf{X}_i , \mathbf{X}_j)}{\sqrt{N^2} + \sqrt{N^2 - 2 M \delta (\mathbf{X}_i , \mathbf{X}_j)}}. \label{eq:em_2}
\end{align}
In order to find the range of Eq. (\ref{eq:em_2}), we first consider its lower bound. Since $2 M \delta (\mathbf{X}_i , \mathbf{X}_j) \geq 0$, we can derive:
\begin{align}
\text{Eq. (\ref{eq:em_2})} \geq \frac{2 M \delta (\mathbf{X}_i , \mathbf{X}_j)}{\sqrt{N^2} + \sqrt{N^2}} = \frac{M}{N} \delta (\mathbf{X}_i , \mathbf{X}_j). \label{eq:lower_bound}
\end{align}
Similarly, the upper bound can be derived as:
\begin{align}
\text{Eq. (\ref{eq:em_2})} \leq \frac{2 M \delta (\mathbf{X}_i , \mathbf{X}_j)}{\sqrt{N^2}} = \frac{2M}{N} \delta (\mathbf{X}_i , \mathbf{X}_j). \label{eq:upper_bound}
\end{align}
By combining (\ref{eq:lower_bound}) and (\ref{eq:upper_bound}), noting that $N = n_i + n_j$ and $M = n_i n_j$, we have the following range of the merging error:
\begin{align}
\frac{n_i n_j}{n_i + n_j} \delta (\mathbf{X}_i , \mathbf{X}_j) \leq \mathcal{E}_{m}(\mathbf{X}_i, \mathbf{X}_j) \leq \frac{2 n_i n_j}{n_i + n_j} \delta (\mathbf{X}_i , \mathbf{X}_j), \label{eq:merging_error_appendix}
\end{align}
which concludes that $\mathcal{E}_{m}(\mathbf{X}_i, \mathbf{X}_j)$ is $\Theta\left(\frac{n_in_j}{n_i + n_j} \delta (\mathbf{X}_i , \mathbf{X}_j)\right)$. \qed

\begin{table*}[t!]
\caption{\textbf{Sensitivity analysis of the batch size.}}
\centering
\resizebox{0.77\textwidth}{!}{%
\begin{tabular}{ccccccccccc
>{\columncolor[HTML]{EFEFEF}}c }
\Xhline{2.0\arrayrulewidth}
Batch Size      & 1     & 2     & 4     & 8     & 16    & 32    & 64    & 128   & 256   & 512   & 1024  \\ \hline
Top-1 Acc. (\%) & 79.64 & 79.64 & 79.69 & 79.68 & 79.79 & 79.64 & 79.68 & 79.70 & 79.76 & 79.67 & 79.76 \\
GFLOPs          & 3.27  & 3.25  & 3.40  & 3.32  & 3.24  & 3.24  & 3.23  & 3.23  & 3.23  & 3.23  & 3.23  \\ \Xhline{2.0\arrayrulewidth}
\end{tabular}%
}
\label{tab:A2}
\end{table*}

\section{Sensitivity Analysis of Batch Sizes} \label{app:sens}
Due to our batch-adaptive merging, the performance can vary across different batch sizes. To examine this, we present a sensitivity analysis on batch sizes in Table \ref{tab:A2}. Similar to Table \ref{tab:A1}, the resulting FLOPs can slightly differ while varying the batch size yet equalizing all the other settings and variables. The results do not show any specific patterns with batch sizes, indicating no clear correlation between batch sizes and model performance. Therefore, we simply adopt the largest batch size of 1024 for practical implementation and computational efficiency.





\section{Full Results with Varying FLOPs}
In Tables \ref{taba1}--\ref{taba4}, we provide full experimental results on accuracy with varying FLOPs, some of which are corresponding to the points in Figure \ref{fig4}. We also present all the optimal hyperparameter values therein. To find these optimal values, we conduct a hyperparameter sweep for $\alpha$, $\beta$, and $\theta_{\text{min}}$ using Weights \& Biases (WandB) tool. We set the ranges from 0.945 to 1.000 for $\alpha$, from 0.015 to 0.050 for $\beta$, and from 0.800 to 0.945 for $\theta_{\text{min}}$, sweeping all parameters with intervals of 0.005. 

In Table \ref{taba1}, we present the full results on DeiT without additional training. With these results, we fine-tune each compressed models for 30 epochs, demonstrating the performance improvement in Table \ref{taba2} for each set. In this case, the target FLOPs may vary slightly between Table \ref{taba1} and \ref{taba2}, considering that our merging scheme uses an adaptive strategy, which affects the number of merged tokens for each layer.

In Table \ref{taba3}, we present the full results on LV-ViT without additional training, and in Table \ref{taba4}, we present the full results for both T2T-ViT and MAE, also without additional training.

\section{Layer-Specific Visualizations} \label{app:vis}
In Figure \ref{figa1}, we present full visualizations of ATM across layers, comparing with ToMe \cite{BolyaFDZFH23}, a static top-\textit{r} merging method. From these results, we observe that ATM merges significantly more tokens in the initial layers than ToMe for simpler samples, thereby effectively reducing computational costs. Conversely, for complex samples, ATM adopts a more gradual merging approach, carefully preserving important image features.

\twocolumn

\begin{table}[t!]
\caption{\textbf{Full results on DeiT without training.}}
\centering
\resizebox{0.9\columnwidth}{!}{%
\begin{tabular}{cccccc}
\Xhline{2.0\arrayrulewidth}
Model                    & $\alpha$ & $\beta$ & $\theta_{\text{min}}$ & GFLOPs & \begin{tabular}[c]{@{}c@{}}Top-1 Acc.\\ (\%)\end{tabular} \\ \hline
\multirow{14}{*}{DeiT-T} & -        & -       & -                     & 1.25   & 72.14                                                     \\  
                         & 0.98    & 0.035   & 0.93                  & 1.05   & 72.19                                                     \\
                         & 0.975    & 0.025   & 0.91                  & 0.98   & 72.15                                                     \\
                         & 0.995        & 0.025   & 0.895                   & 0.96   & 72.14                                                     \\
                         & 0.985    & 0.025    & 0.885                 & 0.93   & 72.08                                                     \\
                         & 1        & 0.035    & 0.88                  & 0.90   & 72.04                                                     \\
                         & 1        & 0.03    & 0.865                 & 0.88   & 72.01                                                     \\
                         & 1    & 0.035    & 0.86                 & 0.85   & 71.93                                                     \\
                         & 0.975     & 0.03   & 0.86                 & 0.83   & 71.83                                                     \\
                         & 0.995     & 0.04   & 0.855                  & 0.81   & 71.74                                                     \\
                         & 0.98        & 0.035   & 0.84                 & 0.78   & 71.58                                                     \\
                         & 0.985    & 0.04    & 0.83                  & 0.75   & 71.33                                                     \\
                         & 0.99     & 0.045   & 0.825                  & 0.72   & 71.08                                                     \\
                         & 0.995        & 0.045   & 0.815                  & 0.71   & 70.94                                                     \\ \hline
\multirow{14}{*}{DeiT-S} & -        & -       & -                     & 4.60   & 79.81                                                     \\
                         & 0.995        & 0.02   & 0.91                  & 3.67   & 79.86                                                     \\
                         & 0.99    & 0.02    & 0.895                  & 3.53   & 79.86                                                     \\
                         & 0.995     & 0.04    & 0.895                 & 3.44   & 79.84                                                     \\
                         & 0.99     & 0.02    & 0.875                  & 3.38   & 79.82                                                     \\
                         & 0.99        & 0.04    & 0.88                  & 3.26   & 79.80                                                     \\
                         & 0.985       & 0.045   & 0.88                 & 3.23  & 79.76
                                        \\
                         & 1    & 0.03   & 0.855                  & 3.13   & 79.74                                                     \\
                         & 1    & 0.04   & 0.86                 & 3.07   & 79.71                                                     \\
                         & 0.99 & 0.04  & 0.86                  & 3.04  & 79.65
                                \\
                         & 1    & 0.05   & 0.855                 & 2.96   & 79.63                                                     \\
                         & 0.995    & 0.05    & 0.845                 & 2.85   & 79.53                                                     \\
                         & 0.975     & 0.04    & 0.835                  & 2.79   & 79.44                                                     \\
                         & 0.97        & 0.045    & 0.83                 & 2.67   & 79.29                                                     \\
                         & 0.98        & 0.05    & 0.825                 & 2.62   & 79.22                                                     \\
                         & 0.96    & 0.035    & 0.81                  & 2.56   & 79.08                                                     \\ \hline
\multirow{17}{*}{DeiT-B} & -        & -       & -                     & 17.57  & 81.79                                                     \\
                         & 0.995        & 0.025   & 0.92                 & 14.60  & 81.84                                                     \\
                         & 0.995        & 0.015    & 0.89                 & 14.26  & 81.79                                                     \\
                         & 1     & 0.015   & 0.8                   & 14.10  & 81.77                                                     \\
                         & 1    & 0.025    & 0.89                  & 13.89  & 81.76                                                     \\
                         & 1     & 0.025   & 0.865                  & 13.24  & 81.68                                                     \\
                         & 0.995     & 0.025   & 0.86                 & 12.99  & 81.61                                                     \\
                         & 0.99     & 0.025   & 0.845                 & 12.41  & 81.51                                                     \\
                         & 0.99     & 0.03    & 0.84                 & 11.95  & 81.38                                                     \\
                         & 0.98    & 0.025   & 0.8                 & 11.21  & 81.05                                                     \\
                         & 0.965     & 0.025   & 0.835                 & 11.05  & 80.94                                                     \\
                         & 0.965     & 0.025    & 0.82                  & 10.78  & 80.79                                                     \\
                         & 0.965    & 0.025    & 0.8                 & 10.49  & 80.60                                                     \\
                         & 0.945    & 0.02    & 0.81                 & 10.29  & 80.45                                                     \\
                         & 0.95    & 0.025   & 0.805                   & 9.98  & 80.17                                                     \\ \Xhline{2.0\arrayrulewidth}
\end{tabular}%
}
\label{taba1}
\end{table}

\begin{table}[t!]
\caption{\textbf{Full results on DeiT with fine-tuning.}}
\centering
\resizebox{0.9\columnwidth}{!}{%
\begin{tabular}{cccccc}
\Xhline{2.0\arrayrulewidth}
Model                    & $\alpha$ & $\beta$ & $\theta_{\text{min}}$ & GFLOPs & \begin{tabular}[c]{@{}c@{}}Top-1 Acc.\\ (\%)\end{tabular} \\ \hline
\multirow{14}{*}{DeiT-T} & -        & -       & -                     & 1.25   & 72.14                                                                              \\ 
                         & 0.98    & 0.035   & 0.93                  & 1.04   & 72.44                                                     \\
                         & 0.975    & 0.025   & 0.91                  & 0.98   & 72.41                                                     \\
                         & 0.995        & 0.025   & 0.895                   & 0.96   & 72.42                                                     \\
                         & 0.985    & 0.025    & 0.885                 & 0.92   & 72.32                                                     \\
                         & 1        & 0.035    & 0.88                  & 0.90   & 72.33                                                     \\
                         & 1        & 0.03    & 0.865                 & 0.88   & 72.33                                                     \\
                         & 1    & 0.035    & 0.86                 & 0.85   & 72.39                                                     \\
                         & 0.975     & 0.03   & 0.86                 & 0.83   & 72.25                                                     \\
                         & 0.995     & 0.04   & 0.855                  & 0.81   & 72.20                                                     \\
                         & 0.98        & 0.035   & 0.84                 & 0.77   & 72.08                                                     \\
                         & 0.985    & 0.04    & 0.83                  & 0.74   & 71.89                                                     \\
                         & 0.99     & 0.045   & 0.825                  & 0.72   & 71.79                                                     \\
                         & 0.995        & 0.045   & 0.815                  & 0.71   & 71.76                                                     \\ \hline
\multirow{15}{*}{DeiT-S} & -        & -       & -                     & 4.60   & 79.81                                                     \\
                         & 0.995        & 0.02   & 0.91                  & 3.64   & 79.92                                                     \\
                         & 0.99    & 0.02    & 0.895                  & 3.50   & 79.94                                                     \\
                         & 0.995     & 0.04    & 0.895                 & 3.39   & 79.92                                                     \\
                         & 0.99     & 0.02    & 0.875                  & 3.35   & 79.90                                                     \\
                         & 0.99        & 0.04    & 0.88                  & 3.23   & 79.97                                                     \\
                         & 0.985       & 0.045   & 0.88                 & 3.19  & 80.00
                                        \\
                         & 1    & 0.03   & 0.855                  & 3.11   & 79.93                                                     \\
                         & 1    & 0.04   & 0.86                 & 3.04   & 79.89                                                     \\
                         & 0.99 & 0.04  & 0.86                  & 3.01  & 79.92
                                \\
                         & 1    & 0.05   & 0.855                 & 2.93   & 79.86                                                     \\
                         & 0.995    & 0.05    & 0.845                 & 2.81   & 79.80                                                     \\
                         & 0.975     & 0.04    & 0.835                  & 2.75   & 79.77                                                     \\
                         & 0.97        & 0.045    & 0.83                 & 2.63   & 79.77                                                     \\
                         & 0.98    & 0.05    & 0.825                  & 2.59   & 79.64                                                     \\
                         & 0.96        & 0.035    & 0.81                 & 2.52   & 79.58                                                     \\ \hline
\multirow{17}{*}{DeiT-B} & -        & -       & -                     & 17.57  & 81.79                                                     \\
                         & 0.995        & 0.025   & 0.92                 & 14.48  & 81.89                                                     \\
                         & 0.995        & 0.015    & 0.89                 & 14.15  & 81.87                                                     \\
                         & 1     & 0.015   & 0.8                   & 14.00  & 81.84                                                     \\
                         & 1    & 0.025    & 0.89                  & 13.71  & 81.84                                                     \\
                         & 1     & 0.025   & 0.865                  & 13.04  & 81.80                                                     \\
                         & 0.995     & 0.025   & 0.86                 & 12.78  & 81.76                                                     \\
                         & 0.99     & 0.025   & 0.845                 & 12.23  & 81.70                                                     \\
                         & 0.99     & 0.03    & 0.84                 & 11.76  & 81.66                                                     \\
                         & 0.98    & 0.025   & 0.8                 & 11.05  & 81.50                                                     \\
                         & 0.965     & 0.025   & 0.835                 & 10.85  & 81.52                                                     \\
                         & 0.965     & 0.025    & 0.82                  & 10.59  & 81.45                                                     \\
                         & 0.965    & 0.025    & 0.8                 & 10.31  & 81.26                                                     \\
                         & 0.945    & 0.02    & 0.81                 & 10.12  & 81.24                                                     \\
                         & 0.95    & 0.025   & 0.805                   & 9.80  & 81.19                                                     \\ \Xhline{2.0\arrayrulewidth}
\end{tabular}%
}
\label{taba2}
\end{table}

\begin{table}[t!]
\caption{\textbf{Full results on LV-ViT without training.}}
\centering
\resizebox{0.9\columnwidth}{!}{%
\begin{tabular}{cccccc}
\Xhline{2.0\arrayrulewidth}
Model                      & $\alpha$ & $\beta$ & $\theta_{\text{min}}$ & GFLOPs & \begin{tabular}[c]{@{}c@{}}Top-1 Acc.\\ (\%)\end{tabular} \\ \hline
\multirow{11}{*}{LV-ViT-T} & -        & -       & -                     & 2.87   & 79.11                                                     \\
                           & 0.96    & 0.02    & 0.935                  & 2.42   & 79.15                                                     \\
                           & 0.985     & 0.015   & 0.92                 & 2.39   & 79.09                                                     \\
                           & 0.96        & 0.045   & 0.915                  & 2.30   & 79.08                                                     \\
                           & 0.98    & 0.035    & 0.905                 & 2.28   & 79.02                                                     \\
                           & 0.955     & 0.05   & 0.905                 & 2.24   & 79.00                                                     \\
                           & 0.97    & 0.025   & 0.885                  & 2.21   & 78.96                                                     \\
                           & 0.96    & 0.045    & 0.885                 & 2.16   & 78.92                                                     \\
                           & 0.995    & 0.045    & 0.865                 & 2.12   & 78.82                                                     \\
                           & 0.955    & 0.035    & 0.865                 & 2.09   & 78.77                                                     \\
                           & 0.95    & 0.045   & 0.865                  & 2.06   & 78.69                                                     \\ 
                           & 0.96   & 0.045 & 0.855                     & 2.05  & 78.68
                                            \\
                            & 0.95  & 0.05  & 0.855                     & 2.02  & 78.55
                                            \\
                            & 0.945 & 0.025 & 0.83                      & 2.00  & 78.47
                                            \\ \hline
\multirow{13}{*}{LV-ViT-S} & -        & -       & -                     & 6.57   & 83.26                                                     \\
                           & 0.995    & 0.05    & 0.945                 & 5.27   & 83.27                                                     \\
                           & 0.955    & 0.015   & 0.935                  & 4.98   & 83.22                                                     \\
                           & 0.95     & 0.015   & 0.925                 & 4.73   & 83.16                                                     \\
                           & 0.955     & 0.015   & 0.915                 & 4.56   & 83.09                                                     \\
                           & 0.955    & 0.035   & 0.905                  & 4.34   & 83.00                                                     \\
                           & 0.995    & 0.04    & 0.9                  & 4.31   & 82.94                                                     \\
                           & 0.995     & 0.035    & 0.89                 & 4.16   & 82.84                                                     \\
                           & 0.945    & 0.05   & 0.89                & 4.01   & 82.79                                                     \\
                           & 0.945    & 0.045    & 0.885                 & 3.93   & 82.74                                                     \\
                           & 0.95     & 0.045   & 0.875                 & 3.80   & 82.60                                                     \\
                           & 0.945    & 0.045   & 0.865                 & 3.64   & 82.37                                                     \\
                           & 0.945    & 0.04    & 0.86                 & 3.59   & 82.22
                                            \\ 
                            & 0.945     & 0.05  & 0.855                 & 3.51  & 82.08
                                            \\ \hline
\multirow{11}{*}{LV-ViT-M} & -        & -       & -                     & 12.74  & 84.00                                                     \\
                           & 0.985        & 0.035    & 0.94                  & 9.00   & 83.95                                                     \\
                           & 0.99    & 0.03   & 0.925                 & 8.27   & 83.89                                                     \\
                           & 0.98     & 0.03    & 0.92                 & 8.04   & 83.84                                                     \\
                           & 0.955    & 0.045    & 0.92                 & 7.95   & 83.83                                                     \\
                           & 0.975    & 0.045    & 0.915                  & 7.82   & 83.80                                                     \\
                           & 0.955        & 0.015    & 0.915                 & 7.77   & 83.79                                                     \\
                           & 0.965     & 0.045   & 0.91                 & 7.60   & 83.77                                                     \\
                           & 0.955    & 0.04    & 0.905                  & 7.35   & 83.73                                                     \\
                           & 0.995     & 0.05   & 0.895                 & 7.06   & 83.64                                                     \\
                           & 0.965    & 0.025    & 0.89                 & 6.96   & 83.60                                                     \\ 
                           & 0.95       & 0.025 & 0.885             & 6.68  & 83.56 
                                            \\
                            & 0.96  & 0.045 & 0.88                  & 6.55  & 83.49
                                            \\
                            & 0.965 & 0.04  & 0.875                 & 6.46  & 83.44
                                            \\
                            & 0.95  & 0.035 & 0.875                 & 6.35  & 83.41
                                            \\
                            & 0.945 & 0.05  & 0.87                  & 6.16  & 83.29
                                            \\
                           \Xhline{2.0\arrayrulewidth}
\end{tabular}%
}
\label{taba3}
\end{table}

\begin{table}[t!]
\caption{\textbf{Full results on T2T-ViT and MAE without training.}}
\centering
\resizebox{0.9\columnwidth}{!}{%
\begin{tabular}{cccccc}
\Xhline{2.0\arrayrulewidth}
Model                                   & $\alpha$ & $\beta$ & $\theta_{\text{min}}$ & GFLOPs & \begin{tabular}[c]{@{}c@{}}Top-1 Acc.\\ (\%)\end{tabular} \\ \hline
\multirow{15}{*}{T2T-ViT$_\text{t}$-14} & -        & -       & -                     & 6.09    & 81.70                                                      \\ 
                                        & 0.995    & 0.015   & 0.945                 & 5.35   & 81.72                                                     \\
                                        & 0.97     & 0.02    & 0.92                  & 4.97   & 81.70                                                     \\
                                        & 0.995    & 0.02    & 0.9                   & 4.76   & 81.66                                                     \\
                                        & 0.96    & 0.035   & 0.89                  & 4.54   & 81.60                                                     \\
                                        & 1    & 0.045    & 0.88                 & 4.45   & 81.56                                                     \\
                                        & 0.955        & 0.03    & 0.875                 & 4.37   & 81.51                                                     \\
                                        & 0.97     & 0.045   & 0.865                  & 4.25   & 81.45                                                     \\
                                        & 0.96     & 0.04   & 0.855                 & 4.15   & 81.37                                                     \\
                                        & 0.975        & 0.04    & 0.85                  & 4.12   & 81.36                                                     \\
                                        & 0.975        & 0.05    & 0.85                  & 4.09   & 81.31                                                     \\
                                        & 0.955    & 0.035   & 0.845                 & 4.05   & 81.28                                                     \\
                                        & 0.975    & 0.045    & 0.84                  & 4.01   & 81.23                                                     \\
                                        & 0.95    & 0.05    & 0.845                  & 3.99   & 81.21                                                     \\
                                        & 0.965    & 0.045   & 0.835                 & 3.96   & 81.16                                                     \\ \hline
\multirow{13}{*}{T2T-ViT$_\text{t}$-19} & -        & -       & -                     & 9.79    & 82.40                                                      \\ 
                                        & 0.985    & 0.045   & 0.94                  & 7.79   & 82.40                                                     \\
                                        & 0.985    & 0.015   & 0.93                  & 7.52   & 82.39                                                     \\
                                        & 0.96    & 0.025    & 0.915                 & 6.97   & 82.34                                                     \\
                                        & 1     & 0.035    & 0.905                  & 6.73   & 82.29                                                     \\
                                        & 0.975    & 0.025    & 0.9                 & 6.59   & 82.27                                                     \\
                                        & 0.955     & 0.04    & 0.895                  & 6.37   & 82.22                                                     \\
                                        & 0.995    & 0.035    & 0.885                 & 6.21   & 82.20                                                     \\
                                        & 0.945    & 0.05    & 0.88                  & 5.95   & 82.07                                                     \\
                                        & 0.96     & 0.04    & 0.87                  & 5.78   & 81.99                                                     \\
                                        & 0.97    & 0.035    & 0.86                  & 5.63   & 81.91                                                     \\
                                        & 0.95    & 0.05   & 0.855                  & 5.42   & 81.83                                                     \\
                                        & 0.945     & 0.045    & 0.85                 & 5.32   & 81.75                                                     \\
                                        & 0.96     & 0.05    & 0.84                 & 5.18   & 81.61                                                     \\
                                        & 0.95     & 0.05    & 0.83                 & 5.00   & 81.47                                                     \\
                                        & 0.955     & 0.045    & 0.82                 & 4.91   & 81.32                                                     \\
                                        & 0.945     & 0.05    & 0.82                 & 4.85   & 81.23                                                     \\
                                        & 0.96     & 0.05    & 0.815                 & 4.83   & 81.17                                                     \\
                                        \hline
\multirow{17}{*}{MAE}                   & -        & -       & -                     & 17.57   & 83.72                                                      \\ 
                                        & 1    & 0.025   & 0.94                  & 16.58  & 83.67                                                     \\
                                        & 0.995        & 0.025   & 0.905                  & 15.69  & 83.56                                                     \\
                                        & 0.99    & 0.035   & 0.92                  & 15.48  & 83.52                                                     \\
                                        & 0.985        & 0.02   & 0.905                 & 15.24  & 83.49                                                     \\
                                        & 0.985        & 0.02    & 0.89                 & 14.87  & 83.43                                                     \\
                                        & 0.995     & 0.03   & 0.87                  & 14.49  & 83.37                                                     \\
                                        & 0.985        & 0.02   & 0.86                 & 14.13  & 83.31                                                     \\
                                        & 0.995    & 0.045    & 0.865                 & 13.74  & 83.24                                                     \\
                                        & 0.97    & 0.015    & 0.845                 & 13.43  & 83.19                                                     \\
                                        & 0.975     & 0.02   & 0.845                  & 13.24  & 83.13                                                     \\
                                        & 0.98    & 0.02    & 0.81                 & 13.05  & 83.07                                                     \\
                                        & 0.975    & 0.03    & 0.845                 & 12.73  & 83.00                                                     \\
                                        & 0.965    & 0.025    & 0.845                 & 12.53  & 82.93                                                     \\
                                        & 0.975    & 0.03   & 0.82                 & 12.05  & 82.83                                                     \\
                                        & 0.975    & 0.03    & 0.815                  & 11.92  & 82.74                                                     \\
                                        & 0.975     & 0.03    & 0.805                 & 11.70  & 82.64                                                     \\ 
                                        & 0.975     & 0.03    & 0.8                 & 11.61  & 82.55                                                     \\ 
                                        & 0.96     & 0.03    & 0.82                 & 11.50  & 82.46                                                     \\ 
                                        \Xhline{2.0\arrayrulewidth}
\end{tabular}%
}
\label{taba4}
\end{table}

\onecolumn

\begin{figure*}[t!]
    \centering
    \includegraphics[width=0.95\textwidth]{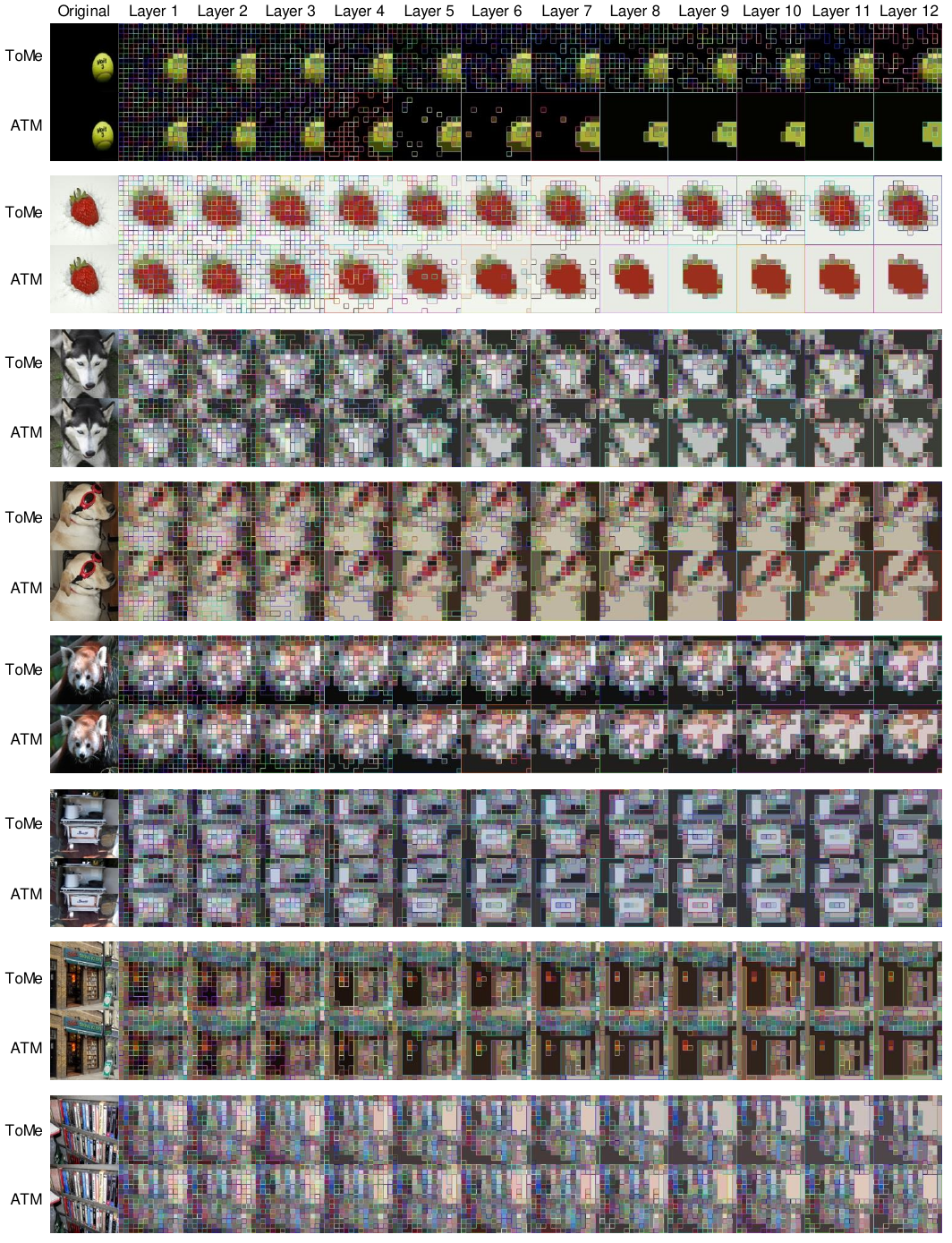} 
    \caption{\textbf{Layer-specific visualizations.}}
    \label{figa1}
\end{figure*}

\end{document}